\documentclass[letterpaper]{article} 
\usepackage{aaai25}  
\usepackage{times}  
\usepackage{helvet}  
\usepackage{courier}  
\usepackage[hyphens]{url}  
\usepackage{graphicx} 
\usepackage{natbib}  
\usepackage{caption} 
\usepackage{algorithm}
\usepackage{algorithmic}

\usepackage{newfloat}
\usepackage{listings}
\DeclareCaptionStyle{ruled}{labelfont=normalfont,labelsep=colon,strut=off} 
\floatstyle{ruled}
\newfloat{listing}{tb}{lst}{}
\floatname{listing}{Listing}
\usepackage{lipsum}
\usepackage{amsfonts}
\usepackage{amsthm}
\usepackage{graphicx}
\usepackage{pgfplots}
\pgfplotsset{compat=1.18}
\usepackage{tikz}
\newtheorem{thm}{Theorem}
\newtheorem{lem}[thm]{Lemma}
\newtheorem{defn}[thm]{Definition}

\newtheorem{ex}[thm]{Example}
\usepackage{epstopdf}
\usepackage{amsopn}

\usepackage{microtype}
\usepackage{subfigure}
\usepackage{amsmath}
\usepackage{booktabs} 

\title{The VOROS: Lifting ROC Curves to 3D to Summarize Unbalanced Classifier Performance}
\author{
    Christopher Ratigan\textsuperscript{\rm 1},  Lenore Cowen\textsuperscript{\rm 1, 2}
}
\affiliations{
    \textsuperscript{\rm 1} Department of Mathematics\\

    \textsuperscript{\rm 2} Department of Computer Science \\
    Tufts University \\
    Medford, MA 02155 USA \\
    christopher.ratigan@tufts.edu, cowen@cs.tufts.edu

}

\begin{document}

\maketitle

\title{The VOROS: Lifting ROC curves to 3D to summarize unbalanced classifier performance\thanks{Submitted to the editors DATE.
\funding{This work was funded by tba}}}









\begin{abstract}

While the area under the ROC curve is perhaps the most common measure that is  used to rank the relative performance of different binary classifiers, longstanding field folklore has noted that it can be a  measure that ill-captures the benefits of different classifiers when either the actual class values or misclassification costs are highly unbalanced between the two  classes. We introduce a new ROC surface, and the VOROS, a volume over this ROC surface, as a natural way to capture these costs, by lifting the  ROC curve to 3D. Compared to previous attempts to generalize the  ROC curve, our formulation also provides a simple and intuitive way to  model the scenario when only ranges, rather than exact values, are known for possible class imbalance and misclassification costs.  

\end{abstract}

\begin{links}
    \link{Code}{https://github.com/ChristopherRatigan/VOROS}
\end{links}

\section{Introduction}

The ROC curve, constructed by plotting the False Positive Rate (FPR) against the True Positive Rate (TPR), is the most canonical way to measure the performance of a binary classifier.

 \begin{defn}
    Given an ROC curve $y=f(x)$, the area under the ROC curve, henceforth AUROC, is defined by

    $$\int_0^1 f(x)dx$$
\end{defn}

As has long been known~\cite{hand2009measuring}, automatically choosing the binary classifier with the higher AUROC makes perfect sense when two rather strong conditions are met: 1) the expected ratio of positives to negatives in the dataset where the classifier is to be deployed is balanced or nearly balanced, and 2) the cost of misclassifying a positive example as negative is equal or nearly equal to the cost of misclassifying a negative example as positive, i.e. that ${C_0}/{C_1} \approx 1$ and ${|\mathcal{P}|}/{|\mathcal{N}|} \approx 1$, where we denote by $\mathcal{P}$ the set of observations that are positive and $\mathcal{N}$ denotes the set of observations that are negative. $C_0$ and $C_1$ denote the cost of misclassifying a case from $\mathcal{N}$ as belonging to $\mathcal{P}$ and the converse, respectively. The area under the ROC curve of a binary classifier $\mathcal{F}$ represents the probability that given two observations $\rho \in \mathcal{P}$ and $\eta \in \mathcal{N}$, $\mathcal{F}$ assigns a higher score to $\rho$ than it does to $\eta$ \cite{bradley1997use}. 

But, if the costs or classes are not balanced, then it is not always sensible to directly compare the AUROCs in this pairwise manner. If one class has many more cases than the other, comparing a random positive and negative example is unlikely to occur in practice. Likewise, if the two types of errors differ in individual costs then it may not be as important that it is possible to separate a randomly chosen positive example from a randomly chosen negative example. 

Observe that when exact information is available on the ratios ${C_0}/{C_1}$ and ${|\mathcal{P}|}/{|\mathcal{N}|}$, it is always possible to find a binary classifier (point) taken from the ROC curve of classifier $\mathcal{F}_1$ that minimizes the total misclassification cost, and a binary classifier taken from the ROC curve of classifier $\mathcal{F}_2$ that minimizes the total misclassification cost, and then choose whichever classifier has lower total misclassification cost. However, as we (see example datasets, below)  and others demonstrate~\cite{provost1998case,hand2009measuring,drummond2000explicitly}), such a cost comparison no longer corresponds to a comparison of the AUROCs. In fact, classifiers with better AUROCs can perform much worse in total misclassification cost in these settings. 

 Much previous work has attempted to deal with the shortcomings of the AUROC by looking at different notions of a weighted AUROC~\cite{mcclish1989analyzing,maurer2020estimating,li2010weighted}; however, most of these formulations require fixing the cost distribution and class imbalance as a prior. (Two notable exceptions are \cite{hernandez2011brier} and the recent paper of  \cite{shao2024weighted}, that takes an interesting (and complementary approach) to solving the problem we address in this paper, namely simultaneously representing different weighted costs)). Versions of weighted AUROC that apply for  a fixed cost distribution and class imbalance do not easily extend to model uncertainty in the cost distribution or only bounds on the class imbalance (though we note that a previous paper of Guo et al. did define a "3D ROC histogram," with a general 3rd axis that could represent either classifier confidence or cost; where the cost version is closest in spirit to the present work~\cite{guo20193d}). 
 
 By instead making something similar to the weighted AUROC a slice of a ROC volume, lifted to 3D, we are  able to model  {\em ranges} of misclassification costs and class imbalance. We introduce a new ROC surface, where the volume integrated over this ROC surface becomes a generalization of the AUROC where misclassification costs are also naturally represented. Besides mathematical elegance, the advantage of our VOROS to previous weighted AUROC approaches is that we also have a very natural way to ask for the better classifier in a regime where the application domain gives only bounds on the expected class imbalances or misclassification costs, and does not assume these values are known exactly. We show that the VOROS is efficient to compute, and show in several benchmark datasets, that it will choose better classifiers, according to how costs and class imbalances are estimated to occur in the environment in which the classifier will be deployed.

We demonstrate the utility of this approach on two standard benchmark classification datasets, the  UCI Wisconsin Breast Cancer Dataset \cite{misc_breast_cancer_wisconsin_(diagnostic)_17} and BUSI \cite{al2020dataset}, as well as a highly imbalanced credit card fraud dataset. 

\section{Background and Definitions} 

\begin{defn}[ROC Space]
Fix a dataset $\mathcal{X}$ consisting of observations assigned the actual class labels positive ($1$) and negative ($0$). For a binary classifier $\mathcal{F}$, the performance of $\mathcal{F}$ on $\mathcal{X}$ is represented as a point in ROC space where $\mathcal{F}$ is mapped to the point $(x,y)$ with $x$ the false positive rate and $y$ the true positive rate of $\mathcal{F}$ on the dataset $\mathcal{X}$.
\end{defn}

An intuitive rule in ROC space is that one classifier does better than another classifier if it is above and to the left as the following definition makes precise.

\begin{defn}
    A point $(x_1,y_1)$ dominates another point $(x_2,y_2)$ in ROC space if it has a lower false positive rate and higher true positive rate, i.e. if $x_1\leq x_2$ and $y_1 \geq y_2$.
\end{defn}


There are a few points in ROC space that are important for our analysis. A perfect classifier, which correctly classifies all observation in $\mathcal{X}$ has ROC coordinates $(0,1)$. Likewise, the point $(1,0)$ represents a model which misclassifies every case. We complete the unit square in ROC space by defining the {\em baseline classifiers}.

\begin{defn}
    The baseline classifiers $\mathcal{B}_0$ and $\mathcal{B}_1$ are the classification models that label all observations as negative or positive respectively.
\end{defn}



 There is a traditional way of going from points in ROC space to a curve for classifiers that can output a numerical score interpreted as the probability that an observation belongs to the positive class. Simply vary a threshold $s$ which divides the scores of the observations into positive and negative classes based on whether the score of a particular example is greater than or less than $s$ \cite{metz1978basic,swets1988measuring}.


\begin{defn}
    Given a classifier $\mathcal{F}: \mathcal{X} \to [0,1]$, let $\delta_{\mathcal{N}}(s)$ be the cumulative distribution of observations in the negative class and $\delta_{\mathcal{P}}(s)$ be that of the positive class. Then, the observed ROC curve for $\mathcal{F}$ is a plot of $(1-\delta_{\mathcal{N}}(s),1-\delta_{\mathcal{P}}(s))$.
\end{defn}

ROC curves inherit the following natural notion of dominance from the domination of individual points.

\begin{defn}
    Given two ROC curves $f$ and $g$, $f$ dominates $g$ if for every point $P$ on the graph of $g$ there is some point $Q$ on the graph of $f$ such that $Q$ dominates $P$ and no point on the graph of $f$ is strictly dominated by any point on the graph of $g$.
\end{defn}


\begin{defn} The Upper Convex Hull of a collection of $n$ points ${\mathcal{C}}=\{(x_i,y_i)\}_{i=1}^n$, in ROC space is the boundary of the convex hull of ${\mathcal{C}}$ together with the additional points $(0,0)$, $(1,1)$ and $(1,0)$ running from $(1,1)$ to $(0,0)$ oriented counterclockwise.
\end{defn}

While the definition of ROC curve intuitively creates a continuous curve in ROC space for continuous distributions $\delta_{\mathcal{N}}$ and $\delta_{\mathcal{P}}$, it is common to take the discrete set of classifiers given by varying the threshold parameter on a finite dataset, connect them piecewise by lines, and call the resulting graph a ``curve'' \cite{drummond2006cost}. In what follows, we will focus on this discretized case, though much of this generalizes to continuous ``fitted'' ROC curves.

We note that for such piecewise linear ROC curves $f$ and $g$, the vertices of $f$ dominate the vertices of $g$ if and only if the graph of $f$ lies above the graph of $g$ over the entire interval $[0,1]$. In this way, domination in ROC space is equivalent to the usual domination of real valued functions. 

\begin{lem}
    The Upper Convex Hull of an ROC curve dominates the ROC curve.
\end{lem}

\begin{proof}
    A full proof is given in \cite{provost2001robust}. 
\end{proof}

It has long been known that in ROC space, it is possible to compare the expected costs of classification models to decide which point on an ROC curve, or a collection of such curves is optimal for a given weighting of the relative numbers and costs of false positives and false negatives. The following definition helps to quantify this notion of cost. 

\begin{defn}
    Let ${\cal X = P \sqcup N}$ be the set of observations to be classified, where $\mathcal{P}$ denotes observations whose actual label is positive and likewise, $\mathcal{N}$ denotes negative observations. We define $C_0$ and $C_1$ to be the costs of an individual false positive and false negative respectively. Additionally, given a point $(x,y)$ in ROC space, we define the expected cost of $(x,y)$ to be
    $$C_0|\mathcal{N}|x+C_1|\mathcal{P}|(1-y)$$
\end{defn}



The point in ROC space with the maximal expected cost is $(1,0)$ which misclassifies every case. We can scale our cost by this maximum so that costs lie in the unit interval putting the costs of classifiers on a normalized scale.

\begin{defn}
The normalized expected cost of a point $(x,y)=(\text{FPR},\text{TPR})$ is given by

$$\text{Cost}(x,y,t)=tx+(1-t)(1-y)$$

where $t=\dfrac{C_0|\mathcal{N}|}{C_0|\mathcal{N}|+C_1|\mathcal{P}|}$ is the portion of cost borne by the false positives.
\end{defn}

This formulation of cost is well-known, though there has not been a universally accepted way to apply it to ROC curves in their entirety \cite{provost1998case}.

\begin{lem}
    Let $P(x_1,y_1)$ and $Q(x_2,y_2)$ be two points in ROC space and let $t \in [0,1]$ be fixed. Then, $\text{Cost}(x_1,y_1,t)=\text{Cost}(x_2,y_2,t)$ if and only if the slope of the line between $P$ and $Q$ is $\dfrac{t}{1-t}$.
\end{lem}

\begin{proof}
    The relevant equation is 
    $$tx_1+(1-t)(1-y_1)=tx_2+(1-t)(1-y_2)$$

    Rearranging yields
    $$t(x_1-x_2)=(1-t)(y_1-y_2)$$

    Dividing by $(1-t)(x_1-x_2)$ on both sides yields the desired result

    $$m=\dfrac{t}{1-t}=\dfrac{y_1-y_2}{x_1-x_2}$$

    Where if $t=1$, we let $m=\infty$ define a vertical line.
\end{proof}

\begin{defn}
    Given a fixed value of $t \in [0,1]$ an iso-performance line through $(x_1,y_1)$ is the line through $(x_1,y_1)$ in ROC space with slope $\dfrac{t}{1-t}$ representing all points with the same cost as $(x_1,y_1)$.
\end{defn}

Intuitively, iso-performance lines are the equivalence classes of points in ROC space with the same cost for a fixed value of $t$. These lines instill a linear order which divides a classifier from better performing classifiers above and to the left of the line (but not necessarily the point) and worse performing classifiers below and to the right of the iso-performance line.

\begin{lem}
    Let $t \in [0,1]$ be fixed and let $y=f(x)$ be an ROC curve. A point $(h,f(h))$ has the minimum cost of any point on the ROC curve if the upper convex hull of the curve lies below and to the right of the iso-performance line through $(h,f(h))$ with slope $m=\dfrac{t}{1-t}$.
\end{lem}

\begin{proof}
     Let $P(h,f(h))$ be such as point, then points with a lower cost than $P$ lie above the line through $P$ with slope $\dfrac{t}{1-t}$, but if $P$ achieves the minimal cost of any point on the ROC curve, all points on the ROC curve must lie below and to the right of the iso-performance line as required.
\end{proof}

The term for these extremal iso-performance lines in convex geometry is ``supporting lines'' and the space of all such lines is the ROC Surface as explained below.

\begin{defn}
    Given a region $R$ in the plane, a supporting line $\ell$ of $R$ is a line intersecting $R$ such that all of $R$ lies on the same side of $\ell$.
\end{defn}

\begin{defn}
    Let $t \in [0,1]$, an \textit{optimal point} on an ROC curve is a point on the upper convex hull that has a supporting line of slope $\dfrac{t}{1-t}$.
\end{defn}

Note, it is possible for an ROC curve to have multiple optimal points. Such points occur when the upper convex hull contains a line segment, any point of which is optimal for $t$ such that $\dfrac{t}{1-t}$ is the slope of the segment.



Not only can supporting lines be used to pick out the optimal classifier for a given cost parameter $t$, they also partition ROC space according to cost. The area of classifiers which cost less provides a natural ordering which is explicitly related to the expected cost. This stands in contrast to the area under the ROC curve defined in the introduction.



\section{Main Contribution}

To connect our notion back to the area under the ROC curve, we now generalize this notion to a cost-aware space by defining the {\em area of lesser classifiers}.  

\begin{defn}
    Given $t\in[0,1]$, and a point $(h,k)$ in ROC space, the area of lesser classifiers of $(h,k)$, dependent on $t$, denoted $A_t(h,k)$ is the area of the points $(h',k')$ in ROC space, such that $(h',k')$ has a greater expected cost than $(h,k)$. Define such $(h',k')$ to be lesser classifiers to $(h,k)$ in this case.
\end{defn}

Note if we choose $(h,k)$ to be an optimal point on our curve, then the region containing lesser classifiers will always contain the area under the ROC curve, but it also typically includes additional area in the region above the ROC curve. As we see with the Credit Card Fraud dataset analyzed below, a  classifier with a worse AUROC can have a better area of lesser classifiers value for ranges of unbalanced classification class sizes or misclassification costs.

We are now ready to define the ROC surface, which will allow us to look at the area of lesser classifiers in particular ranges of class imbalance or cost imbalance. 

\begin{defn}
    The ROC surface, ROS, over a point $(h,k)$ in ROC space is the surface, in $(x,y,t)$-space, given by

    $$y=\dfrac{t}{1-t} (x-h)+k$$

    where $t \in [0,1]$
\end{defn}

The ROS parameterizes all lines through $(h,k)$ with nonnegative (possibly infinite) slope.

\begin{lem}
    The ROC surface over a point $(h,k)$ is a saddle surface with the equation

    $$t = \dfrac{Y}{X+Y}$$

    where $Y=y-k$ and $X=x-h$ are simply horizontal shifts of the ROC surface over the origin.
\end{lem}

\begin{proof}
    The proof follows by solving for $t$ in the formula for the ROC surface. 
\end{proof}

The ROC surface provides a natural coordinate system for assessing relative costs over a range of parameter values $t$ since if $t$ is a fixed constant, then for the ROC points $(x_1,y_1)$ and $(x_2,y_2)$, we need only compare the parallel lines 

$$y=\dfrac{t}{1-t}(x-x_1)+y_1 \text{ and } y=\dfrac{t}{1-t}(x-x_2)+y_2$$

In particular, setting $x=0$ yields $y_i - \dfrac{t}{1-t} x_i$. The classifier with the lower cost at this threshold then has a higher value of $y_i - \dfrac{t}{1-t} x_i$.

\begin{defn}
    The ROC Surface associated to an ROC curve $y=f(x)$, is the surface in $(x,y,t)$-space which for each value of $t$ is the maximum of the lines in ROC space through points $(p,f(p))$ with slope $\dfrac{t}{1-t}$. 
\end{defn}

This surface divides the unit cube, $[0,1]^3$, into two sets. When viewed from the vertical line $x=0,y=1$, points $P(x,y,t)$ behind the ROC surface map to points $(x,y)$ in ROC space such that the cost of $(x,y)$ is greater than that of the optimal point on the ROC curve at that $t$ value. Similarly, points $P(x,y,t)$ in front of the ROC surface represent points which perform better than the optimal point on the ROC curve at the associated $t$ value.

This gives a natural notion of performance for the whole ROC curve: the volume over the ROC surface.

\begin{defn}
    The volume over the ROC Surface, VOROS, is defined by
    $$VOROS(f(x))=\int_0^1 \max_{x \in [0,1]}(A_t(x,f(x)))dt$$

    Where $\text{max}(A_t(x,f(x)))$ is the maximal area of lesser classifiers for any point on the ROC curve taken for each $t$.
\end{defn}

One of the useful properties of VOROS is that it can be readily modified to handle ranges and distributions over the relative costs of misclassifications. This is useful because in many applications, it is reasonable to assume one might have bounds on the expected cost imbalances, but not know them exactly. We can capture this uncertainty in our VOROS measure as follows: 

\begin{defn}
    Let $\mu$ be a probability measure for $t \in [0,1]$, let $[a,b] \subset [0,1]$ and let $A_t(x,f(x))$ be the area of lesser classifiers associated to a point on the ROC curve $y=f(x)$, then 
    \begin{multline*}
        VOROS(f(x),[a,b],\mu)\\=\dfrac{1}{\mu([a,b])} \int \max_{x \in [0,1]}(A_t(x,f(x))) d\mu(t)
        \end{multline*}
\end{defn}

This is a generalization of our initial definition of VOROS (Definition~20) since $VOROS(f(x),[0,1], \mu) = VOROS(f(x))$ if $\mu$ is the standard measure on $\mathbb{R}$. In what follows we omit $\mu$ assuming it to be the standard measure on $\mathbb{R}$.

\begin{lem}
    Given an ROC point $(h,k)$, and fixed costs $(t,1-t)$, if $(h,k)$ outperforms both baseline classifiers, then the area of lesser classifiers for $(h,k)$ has the general form below.
    
    $$A_t(h,k)=1+\dfrac{(1-k)^2}{2}+\dfrac{h^2}{2}-h(1-k)-\dfrac{(1-k)^2}{2t}-\dfrac{h^2}{2(1-t)}$$

\end{lem}    

\begin{proof} 

\begin{figure}
    \centering
    \includegraphics{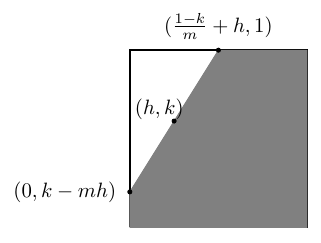}
    \caption{The area of Lesser Classifiers for the point $(h,k)$ lies below the iso-performance line with slope $m=\frac{t}{1-t}$.}
    \label{lesser}
\end{figure}

Let $m=\frac{t}{1-t}$ be the slope of the iso-performance line. First, note that since $(h,k)$ outperforms the baseline classifiers, the area of lesser classifiers will be of the general form of the shaded region in Figure~\ref{lesser}.

This is because the line bounding this region is above and to the left of the parallel lines through $(0,0)$ and $(1,1)$. Thus, the dividing line for lesser classifiers must intersect the $y$-axis above $(0,0)$ and intersect the line $y=1$ to the left of $(1,1)$.

As such, the area of this region is given by one minus the area of the white triangle. The horizontal base of the triangle is $B = \frac{1-k}{m} +h$ while the height of the triangle is $H= 1-k+mh$. Hence, the area of the shaded region is

\begin{align*}
    A_t(h,k) &= 1-\dfrac{1}{2} BH\\
    &= 1- \frac{1}{2}\left(\dfrac{1-k}{m}+h\right) \left(1-k+mh\right)\\
    &= 1-\frac{1}{2} \left(\dfrac{(1-k)^2}{m}+2(1-k)h+mh^2 \right)\\
    &= 1-\dfrac{(1-k)^2}{2m} -h(1-k)-\dfrac{mh^2}{2}\\
\end{align*}

Since $m=\dfrac{t}{1-t}=-1+\dfrac{1}{1-t}$, the area of lesser classifiers is then

$$1+\frac{(1-k)^2}{2}+\frac{h^2}{2}-h(1-k)-\frac{(1-k)^2}{2t}-\frac{h^2}{2(1-t)}$$
\end{proof}

Integrating this last expression yields the following result

\begin{lem}
    Given a vertex $v=(h,k)$ from the upper convex hull of an ROC curve $\{x_i,y_i\}_{i=0}^n$. If $v$ is an optimal point on the cost interval $[a,b] \subset [0,1]$, then, the volume of lesser classifiers for the ROC curve on the interval $(a,b)$ is given by

    \begin{multline*}
        \frac{1}{b-a} \left(1+\dfrac{(1-k)^2}{2}+\dfrac{h^2}{2}-h(1-k) \right)\\ +\dfrac{(1-k)^2}{2} \text{ln}\left(\dfrac{1-b}{1-a} \right) - \dfrac{h^2}{2} \text{ln}\left(\dfrac{b}{a} \right)
    \end{multline*}
    \label{explicit}
\end{lem}

Note that Lemma~\ref{explicit} gives an explicit way to compute the VOROS between points in the upper convex hull of ROC space, which will contain at most $n$ points. Thus the VOROS can be calculated explicitly in time on the order  of the time it takes to calculate the convex hull of $n$ points. In our implementation we use the Quickhull algorithm~\cite{bykat1978convex,eddy1977new} which runs in $O(n \log n)$ expected time. 

There is a direct relationship between the minimum normalized expected cost of a classifier and the volume integrated over the ROC surface.

\begin{thm}
    Let $(h,k)$ be the optimal point on an ROC curve for $t \in [a,b]$. If $c_t(h,k)$ is the normalized expected cost of this point, then the volume over the ROC surface is simply

    $$\dfrac{1}{b-a}\int_a^b 1-\dfrac{\big(c_t(h,k)\big)^2}{2t(1-t)}dt$$
\end{thm}

\begin{proof} 
Since the volume over the ROC surface is simply the integral of the area bounded by a classifier at each threshold, all we need to show is that for each value of $t$, we have the formula

$$A_t(h,k) = 1-\dfrac{c_t(h,k)^2}{2t(1-t)}$$

From earlier, we know that the left hand side is

$$A_t(h,k)=1+\dfrac{(1-k)^2}{2}+\dfrac{h^2}{2}-h(1-k)+\dfrac{(1-k)^2}{2t} -\dfrac{h^2}{2(1-t)}$$

The right hand side is simply,
\begin{multline*}
    1-\dfrac{c_t(h,k)^2}{2t(1-t)} = 1-\dfrac{(th+(1-t)(1-k))^2}{2t(1-t)}\\
    =1 - \dfrac{t^2h^2+2t(1-t)h(1-k)+(1-t)^2(1-k)^2}{2t(1-t)}\\
    =1 - \left(\dfrac{th^2}{2(1-t)}+h(1-k)+\dfrac{(1-t)(1-k)^2}{2t} \right)\\
    =1 - \left(\dfrac{h^2}{2(1-t)}+h(1-k)+\dfrac{(1-k)^2}{2t}-\dfrac{(1-k)^2}{2} \right)
\end{multline*}

where in the last equation we used that $\dfrac{t}{1-t}=-1+\dfrac{1}{1-t}$. From here distributing the negative and rearranging terms yields the desired result.
\end{proof}

This theorem shows an important feature of the VOROS. Given two ROC Surfaces  $S_1$ and $S_2$ associated to ROC curves $f_1$ and $f_2$, if for all $t \in [a,b] \subseteq [0,1]$ we have that $S_1$ has a lower cost than $S_2$, then $\text{VOROS}(f_1,[a,b],\mu) > \text{VOROS}(f_2,[a,b],\mu)$ since the associated integrands have the same inequality. 

There is a natural probabilistic interpretation of the VOROS: it is the probability that choosing $x,y\in [0,1]$ uniformly and independently at random and $t$ independently according to $\mu$, the randomly chosen point $(x,y)$ has a higher cost than an optimal point of the ROC curve. This is particularly useful for the problem of comparing performance of classifiers in cost or class imbalanced problems because if we can estimate the cost or class imbalance, we get a more accurate measure of classifier performance than from the AUROC alone.

For example, the AUROC of the hybrid classifier consisting of the two baselines $\mathcal{B}_0$ and $\mathcal{B}_1$ is $\frac{1}{2}$, but picking the better of the baseline classifiers will always yield a classification model which has less than $\frac{1}{2}$ the maximal expected cost so long as false positives and false negatives are unequal in aggregate cost. 

\begin{ex}
    Consider the ROC curve consisting of the vertices $\{(0,0),(1,1)\}$. The Volume over the ROC surface for this curve is $\frac{3}{2}-\text{ln}(2)\approx .807$.
\end{ex}

    \begin{figure}
    \centering
    \includegraphics{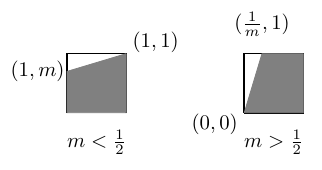}
    \caption{The two possibilities for the area of lesser classifiers for the better of the baselines.}
    \label{fig:base_areas}
    \end{figure}

    \begin{figure}
    \centering
    \includegraphics{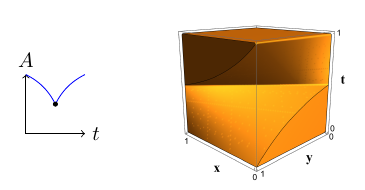}
        \caption{(Left) Graph of $A_t(\{(0,0),(1,1)\})$. (Right) The Volume bounded by these areas.}
        \label{fig: Area t}
    \end{figure}

Set $m=\frac{t}{1-t}$. Then picking the better classifier gives us either of the areas in Figure~\ref{fig:base_areas}. Either way, the area of lesser classifiers is one minus the area of the white triangle.

   $$ A_t = 1-\dfrac{1}{2} BH =\begin{cases} 1-\dfrac{1}{2m} & m>\frac{1}{2}\\
    \\
    1-\dfrac{m}{2} & m<\frac{1}{2} \end{cases}$$

Substituting the expression $m=\frac{t}{1-t}$ gives us the integral.

$$\int_0^1 A_t dt = \int_0^{\frac{1}{2}}1-\frac{t}{2(1-t)} dt + \int_{\frac{1}{2}}^1 1-\frac{(1-t)}{2t} dt$$

A graph of $A_t\left(\{(0,0),(1,1)\}\right)$ as a function of $t$ next to the volume as a region in $(x,y,t)$-space is shown in Figure~\ref{fig: Area t}. The minimum value of the graph of $A_t$ is $0.5$, which occurs when the positive and negative classes have equal aggregate cost, i.e. $t=1-t=\frac{1}{2}$. The value of VOROS as a performance measure stems from the way in which the graph of $A_t$ approaches $1$ on both ends of the interval $[0,1]$: The more imbalanced the classification problem, the less room there is to improve over the better of the two baselines.

\section{Application to the Wisconsin Breast Cancer dataset}

We first show the application of the VOROS to a classic dataset from the UCI repository, namely, \cite{misc_breast_cancer_wisconsin_(diagnostic)_17}. This is a standard binary classification benchmark dataset with a moderate class imbalance (357 positive observations to 212 negative examples). We directly apply logistic regression, naive Bayes and random forest classifiers (all using sklearn default parameters) to a stratified train/test split. Table~\ref{tab:wisc_table_1} summarizes the overall AUROC and VOROS of each classifier without any known bounds on the aggregate cost ratio on the test set. 

\begin{table}
\begin{center}
\begin{tabular}{|c|c|c|}
\hline
 & VOROS([0,1]) & AUROC\\
 \hline
    Log & 99.8\% & 99.3\%\\
    \hline
    Forest & 99.8\% & 99.0\%\\
    \hline
    Naive & 98.4\% & 95.1\%\\
    \hline
    Baseline & 80.7\% & 50\%\\
    \hline
    
\end{tabular}
\caption{Table of VOROS and AUROC values for classifiers trained on the Wisconsin Breast Cancer dataset.}
\label{tab:wisc_table_1}
\end{center}
\end{table}

Notice that the area under the ROC curve is always less than that of the corresponding volume over the ROC surface. While naively, it may be tempting to claim that both VOROS and AUROC rank the classifiers as $\text{Log}>\text{Forest}>\text{Naive}$, the situation is a bit more subtle than this as shown in Figure~\ref{fig:Wisconsin}.

\begin{figure}
\centering
\includegraphics[width=.35\textwidth]{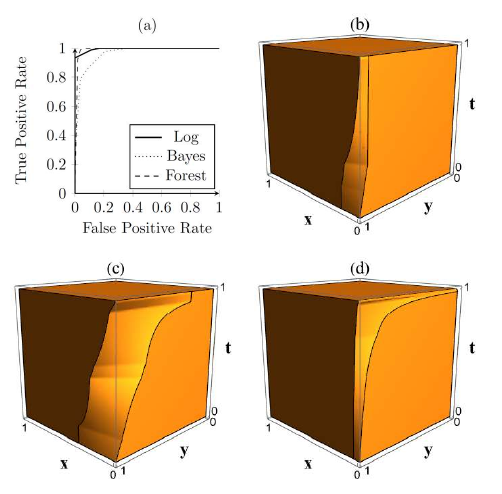}
\caption{(a) A plot of the ROC curves  for the Wisconsin Breast Cancer dataset. (b) The VOROS for the Logistic Regression classifier. (c) The VOROS for the Bayes classifier. (d) The VOROS for the Random Forest classifier.}
\label{fig:Wisconsin}
\end{figure}



Here, we see that our logistic regression and random forest classifier each dominate the naive Bayes model in ROC space. Thus, both logistic regression and random forest do better than naive Bayes for all choices of cost/class imbalance. However, logistic regression and random forest cannot be ranked so definitively since the ROC curves cross \cite{hand2001simple}. Which classifier does best will depend on the costs.

\begin{table}
\begin{center}
\begin{tabular}{|c|c|c|}
\hline
     & VOROS([0,0.25]) & VOROS(0.75,1)\\
     \hline
    Log & 99.8\% & 99.9\%\\
    \hline
    Forest & 99.9\% & 99.6\%\\
    \hline
    Baseline & 92.5\% & 92.5\%\\
    \hline
\end{tabular}
\caption{Table summarizing the VOROS for two different cost ranges on the Wisconsin Breast Cancer dataset.}
\label{tab:wisc_table_2}
\end{center}
\end{table}
 Thus, in ROC space neither curve dominates the other. This can be addressed by VOROS if we are given bounds on the parameter $t$. For example, if we take $t$ to be in the range $0 \leq t \leq 0.25$, then we get that random forest outperforms logistic regression. On the other hand, if $t$ were in the range $0.75 \leq t \leq 1$, then we get the opposite ranking as shown in Table~\ref{tab:wisc_table_2}. Recall that these percentages represent the portion of the ROC points which are  lesser classifiers in this range of cost values. In particular, over this range, picking the better of the two baselines already outperforms over 92\% of ROC points.

\section{Application to BUSI}

We also consider a more modern breast cancer dataset, the Breast UltraSound Images dataset, henceforth BUSI~\cite{al2020dataset}. This dataset consists of full ultrasound breast images, from 600 women patients.  It is actually a 3-class dataset. So, to make it a binary classification task, we consider  both ``normal'' and ``benign'' to be the negative class and ``malignant'' to  be the positive class. 

We encode the data using Pytorch's default vision transformer after cropping the images and scaling them to be the same size. The resulting numerical features are then used to fit logistic regression (with $C=1000$ and $max\_iter=25$), naive Bayes (sklearn default), and random forest classifiers (sklearn default). 

\begin{table}
\begin{center}
\begin{tabular}{|c|c|c|}
\hline
 & VOROS([0,1]) & AUROC\\
 \hline
    Log & 97.2\% & 91.8\%\\
    \hline
    Forest & 97.5\% & 93.3\%\\
    \hline
    Naive & 95.4\% & 88.3\%\\
    \hline
    Baseline & 80.7\% & 50\%\\
    \hline
    
\end{tabular}
\caption{Table summarizing the VOROS and AUROC values for classifiers trained on the BUSI dataset.}
\label{tab:BUSI}
\end{center}
\end{table}

Table~\ref{tab:BUSI} summarizes the results. As we have seen before, the area under the ROC curve is less than that of the corresponding volume over the ROC surface. The ranking of nonbaseline classifiers according to AUROC is $\text{Forest}>\text{Log}>\text{Naive}$, which is the same ranking for the VOROS over the entire interval $[0,1]$, however, plotting the ROC curves, a different story emerges.









In Figure~\ref{fig:BUSI curves} all three curves intersect, hence no single model dominates, and which of these three models performs better will depend on the specific cost and class imbalances. For example, Table~\ref{tab:BUSI2} summarizes the VOROS for a few different choices of cost intervals.

\begin{table}
\begin{center}
\begin{tabular}{|c|c|c|c|}
\hline
     & V([0,1/3]) & V(1/3,2/3) & V(2/3,1)\\
     \hline
    Log & 96.5\% & 96.7\% & 98.4\%\\
    \hline
    Forest & 98.0\% & 96.7\% & 98.0\%\\
    \hline
    Naive & 96.4\% & 93.5\% & 96.1\%\\ 
    \hline
    Baseline & 89.2\% & 63.7\% & 89.2\%\\
    \hline
\end{tabular}
\caption{Table summarizing the VOROS over different ranges of costs for the BUSI dataset. Note we abbreviate VOROS to simply V due to space constrains.}
\label{tab:BUSI2}
\end{center}
\end{table}


While in Figure~\ref{fig:BUSI curves} it seems the random forest model performs best, Table~\ref{tab:BUSI2} shows that the logistic regression model performs better for cases where the aggregate cost of false positives is much higher than that of false negatives. This can be confirmed by a more careful inspection of Figure~\ref{fig:BUSI curves} where the solid curve crosses paths with the dotted curve. There is a range of steeply sloped lines where the solid curve will perform better and these contribute to its better performance in the VOROS for the rightmost column above.


Overall, we see that how classifiers should be ranked depends strongly on the relative aggregate costs of false positives and false negatives on these data.

\begin{figure}
\centering
\includegraphics[width=.40\textwidth]{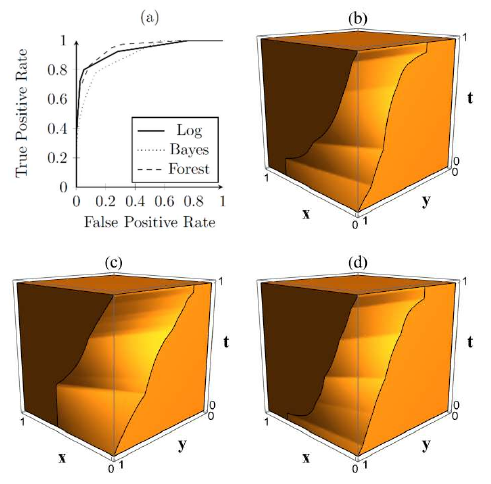}
\caption{(a) A plot of the ROC curves  for the BUSI dataset. (b) The VOROS for the Logistic Regression classifier. (c) The VOROS for the Bayes classifier. (d) The VOROS for the Random Forest classifier.}
\label{fig:BUSI curves}
\end{figure}

\section*{Application to a credit fraud dataset}

The previous two examples have both very good absolute performance of the classifiers considered, plus the class imbalance is not so extreme. The biggest strength of the VOROS, however, is when there is extreme imbalance in class membership and misclassification costs, which we demonstrate in an example credit card fraud dataset. We use a processed dataset derived from a set of  transactions made by credit cards in September 2013 by European cardholders over the course of two days. \cite{dal2015calibrating,dal2014learned,dal2017credit,dal2015adaptive,carcillo2018scarff,carcillo2018streaming,lebichot2020deep,carcillo2021combining,le2022reproducible,lebichot2021incremental} There are a total of 492 frauds (positive class)  out of 284,807 transactions, so the fraud class accounts for only about $.172\%$ of the transactions. The dataset is published with derived features which are  a set of principal components from PCA processing the original features of the dataset (since the original features could not be published for confidentiality concerns). 

We build logistic regression models: one using just the first principal component ($M_1$) and the other using just the second principal component ($M_2$) to model whether transactions were fraudulent. This results in the ROC curves shown in Figure~\ref{fig:fraud_curves}.



\begin{figure}
{\centering
\includegraphics[width=.30\textwidth]{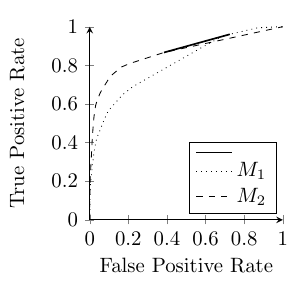}
\caption{A plot of the ROC curves for the Dal Pozzolo Credit Fraud dataset. The solid line connects the optimal points on each curve that have the same cost for a specific value of $t$.}
\label{fig:fraud_curves}
}
\end{figure}

Since these curves intersect, which curve performs better depends on the aggregate misclassification costs. Figure~\ref{fig:fraud_curves} shows the resulting ROC convex hulls together with a line showing the optimal points on each curve that have the same expected cost for a particular choice of aggregate class imbalance. In this case $t\approx 0.2165$, results in the these optimal points having the same normalized expected costs. This is the dividing threshold between $M_1$ performing better and $M_2$ performing better.

There is no guidance given in the dataset for the costs of a false positive versus a false negative, but if we assume that the cost of incorrectly flagging a true transaction as fraud is small (requiring more testing and verification, or customer annoyance to get a flagged card unfrozen) compared to the cost of missing fraud, then the misclassification cost of not detecting fraud could vary depending on both the number of fraudulent transactions and the dollar amount charged.  Suppose we want to choose between these two classifiers, $M_1$ and $M_2$, when we have only bounds on these values in the real world: for example, we assume  that between $0.1\%$ and $1\%$ of transactions are fraudulent, and that the cost of the false negatives is between $500$ and $5000$ times the cost of a false positive. Using the AUROC, we would pick $M_2$, but using the VOROS we find that $\dfrac{1-t}{t}=\dfrac{C_0|N|}{C_1|P|}=\dfrac{C_0}{C_1} \dfrac{|N|}{|P|}$ is the product of the cost and class ratios. Which satisfy the following bounds

$$99 \leq \dfrac{|N|}{|P|} \leq 999 \text{ and } \dfrac{1}{5000}\leq \dfrac{C_0}{C_1} \leq \dfrac{1}{500}$$

These bounds then translate into the following inequality

$$\dfrac{99}{500} \leq \dfrac{t}{1-t} \leq \dfrac{999}{5000}$$

Which since $t=\dfrac{\frac{t}{1-t}}{\frac{t}{1-t}+1}$, induces the limits

$$\dfrac{99}{599} \leq t \leq \dfrac{999}{5999}$$

\begin{table}
\begin{center}
\begin{tabular}{|c|c|c|}
\hline
     & $V([999/5999,99/399])$ & AUROC \\
     \hline
   $M_1$ & 90.1\% & $79.5\%$\\ 
    \hline
    $M_2$ & 89.6 \% & $85.3\%$\\ 
    \hline
    Base & 86.9\% & $50\%$\\ 
    \hline
\end{tabular}
\caption{Table summarizing the VOROS vs. AUROC in the chosen cost range for the Dal Pozzolo Credit Fraud data. Note we abbreviate VOROS to simply V due to space constrains. }
\label{tab:fraud_voros}
\end{center}
\end{table}

Taking the Volume over each ROC surface yields the results in Table~\ref{tab:fraud_voros}. This VOROS comparison immediately shows us facts about classifier performance in our estimated interval that cannot be obtained from comparing the AUROCs. First, we learn that the baseline classifier, that flags {\em all} transactions as fraud actually is not doing so badly in this range by comparing its VOROS (86.9\%) to the values for $M_1$ and $M_2$ (where the strength of a baseline classifier for problems with very rare or very costly classes is high, but is not immediately obvious from looking at ROC curves on their own. In this scenario, most ML papers shift to AUPRC or other measures of classifier performance as a result, so as to not recommend a classifier that is worse than a baseline). Second, comparing the AUROCs in Table~\ref{tab:fraud_voros}, one would expect classifier $M_2$ to dominate, but in this cost range, $M_1$ is  better, and since the   VOROS is larger in this region, for this range we should pick classifier $M_1$ instead of either $M_2$ or the baseline. This agrees with a cost analysis which shows that if the slope of the supporting line to the dotted and dashed curves in Figure~\ref{fig:fraud_curves} are between $\frac{99}{500}$ and $\frac{999}{5000}$, then the supporting line for the dotted curve is higher than that of the dashed curve.

\begin{figure}
{\centering
\includegraphics[width=.40\textwidth]{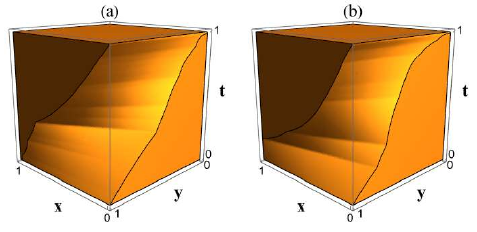}
\caption{The VOROS for logistic regression classifiers (a) using the first principal component, $M_1$, and (b) the second principal component $M_2$.} 
\label{fig:fraud_voros}
}
\end{figure}
\section{Discussion} 

We introduced the VOROS, a natural and easily-computable measure that generalizes the AUROC. We showed both theoretically and in example datasets that it has utility for comparing classifiers when classes or costs are unbalanced. We do note that an effective use of the VOROS is limited to domains where subject matter experts can put some bounds on expected misclassification costs.  So far, our work handles only binary classification tasks.  In fact, it is important to remark that the Volume {\em under} the ROC surface, or VUS, studied by \cite{ferri2003volume} and others \cite{he2008meaning, waegeman2008roc,kang2013estimation}  is an entirely different concept, generalizing the AUROC to  multiple classes, rather than generalizing binary classifiers to different cost and class imbalances.  Indeed,  a generalization our VOROS work to also handle  multiple classes, perhaps by increasing the dimension of the surface, remains a possible topic of future research, as is a generalization to incorporate varied misclassification costs by instance as in \cite{fawcett2006roc}. 


\section{Ethics Statement}  In many application domains (including breast cancer screening), there are ethical implications for how one measures the cost of a false positive versus a false negative.  Once the costs, including ethical costs, of misclassification have been considered by experts, we can analyze the results in our framework. 

\bibliography{references}

\end{document}